\documentclass{article}
\usepackage{microtype}
\usepackage{graphicx}
\usepackage{subfigure}
\usepackage{booktabs} 

\usepackage{hyperref}

\usepackage[accepted]{icml2019}

\icmltitlerunning{Discovering Latent Covariance Structures for Multiple Time Series}

\usepackage{amsmath,amsthm,amssymb}

\usepackage{comment}
\usepackage{caption}
\usepackage{preamble}

\usepackage{accents}

\usepackage[acronym]{glossaries}
\newacronym[longplural={Gaussian Processes}]{gp}{GP}{Gaussian Process}
\newacronym[longplural={Indian Buffet Processes}]{ibp}{IBP}{Indian Buffet Process}
\newacronym{abcd}{ABCD}{Automatic Bayesian Covariance Discovery}
\newacronym{ckl}{CKL}{Compositional Kernel Learning}
\newacronym{rabcd}{R-ABCD}{Relational Automatic Bayesian Covariance Discovery}
\newacronym[longplural={Latent GP Feature Models}]{lgp}{gpLFM}{Latent GP Feature Model}
\newacronym[longplural={Latent Kernel Models}]{lkm}{LKM}{Latent Kernel Model}
\newacronym{vi}{VI}{variational inference}
\newacronym{elbo}{ELBO}{Evidence Lower Bound}
\newacronym{bic}{BIC}{Bayesian Infomation Criteria}
\newacronym{pe}{PSE}{partial set expansion}
\newacronym{fe}{FE}{full expansion}

\usepackage{tikz}
\usetikzlibrary{calc,shapes.geometric,arrows,fit,matrix,positioning,pgfplots.groupplots,trees,patterns}
\usetikzlibrary{arrows,shapes,snakes,automata,backgrounds,petri}
\usetikzlibrary{decorations.text,decorations.pathreplacing}

\newcommand{\todo}[1]{\textcolor{red}{\emph{[TODO: #1]}}}

\definecolor{c1}{rgb}{1.0, 0.4980392156862745, 0.05490196078313725}
\definecolor{c2}{rgb}{0.17254901960784313, 0.6274509803921569, 0.17254901960784313}
\definecolor{c3}{rgb}{0.8392156862745098, 0.15294117647058825, 0.1568627450980392}
\definecolor{c4}{rgb}{0.5803921568627451, 0.403921568627451, 0.7411764705882353}
\definecolor{c5}{rgb}{0.5490196078431373, 0.33725490196078434, 0.29411764705882354}

\begin{document}
\newcommand{\mytitle}{Discovering Latent Covariance Structures for Multiple Time Series}
\twocolumn[
\icmltitle{\mytitle}



\icmlsetsymbol{equal}{*}

\begin{icmlauthorlist}
	\icmlauthor{Anh Tong}{unist}
	\icmlauthor{Jaesik Choi}{unist}
\end{icmlauthorlist}

\icmlaffiliation{unist}{Department of Computer Science and
Engineering, Ulsan National Institute of Science and Technology, Ulsan, 44919, South Korea}

\icmlcorrespondingauthor{Jaesik Choi}{jaesik@unist.ac.kr}

\icmlkeywords{Gaussian Process, The Automatic Statistician, Time series}

\vskip 0.3in
]



\printAffiliationsAndNotice{}
\begin{abstract}
Analyzing multivariate time series data is important to predict future events and changes of complex systems in finance, manufacturing, and administrative decisions. The expressiveness power of Gaussian Process (GP) regression methods has been significantly improved by compositional covariance structures. In this paper, we present a new GP model which naturally handles multiple time series by placing an Indian Buffet Process (IBP) prior on the presence of \textit{shared} kernels. Our selective covariance structure decomposition allows exploiting shared parameters over a set of multiple, selected time series. We also investigate the well-definedness of the models when infinite latent components are introduced. We present a pragmatic search algorithm which explores a larger structure space efficiently. Experiments conducted on five real-world data sets demonstrate that our new model outperforms existing methods in term of structure discoveries and predictive performances.
\end{abstract}

\section{Introduction}
Time series data analysis is important for numerous real-world applications: signal processing of audio and video data; the study of financial variables such as stocks, currencies, and crude oil prices. When several data sources are correlated, a model that exploits a group structure often demonstrates competitive predictive performance~\cite{group_lasso}. It is critical to learn how multiple time series are correlated. Many practical applications i.e. visualizing, filtering or generating reports from multiple time series, depend on their inherent encoded relations. However, it is non-trivial to extract such important relations among them.

A recent work contributed a highly general framework called the~\gls*{abcd} which solves regression tasks using~\gls*{gp} models~\cite{DuvLloGroetal13,Lloyd2014ABCD,Ghahramani15_nature,Hwangb16,BO_Model_NIPS16,KimT17}. Previously, selecting \acrshort*{gp} kernels was heavily based on expert knowledge or trial-and-error. The \acrshort*{abcd} automatically extracts an appropriate compositional covariance structure to fit data based on grammar rules; then it generates human-friendly reports explaining data. 
The compositional covariance structure makes the GP models more expressive and interpretable so that GP kernels are explained in a form of natural language. There are cognitive studies~\cite{compositional_nips2016,SCHULZ201744} showing that compositional functions are intuitively preferred by humans.
Exploiting these key properties of compositional kernel, we develop a kernel composition framework for multiple time series which produces explainable outputs with improved predictive accuracy. 

A solid foundation for multi-task~\acrshort*{gp} regression methods has been established in~\cite{BonillaCW07,spike_lab,kernel_vector,gprn_wilson12icml,GuarnizoAO15_IBP_GP}. However, assigning compositional kernel structures has not yet been investigated in the existing multi-task~\acrshort*{gp} regression methods. Notably, the multi-output GP regression network (GPRN)~\cite{gprn_wilson12icml} is highly general, and models data by the combinations of latent GP functions and weights which are also GPs. Applying structure search is challenging due to the huge search space to cover the whole network.
In order to select appropriate covariance structures for multiple correlated sequences, we model time series by additive structures which are, instead of staying fixed, searched over a set of kernels. We place~\gls*{ibp}~\cite{IBP_GriffithsG05,IBP11} prior over an indicator matrix that represents whether the time series share one or many of these additive kernels. Furthermore, we introduce a search algorithm which enables us to explore a large kernel space. 

Here, we present a new model to handle heterogeneous, correlated multiple time series by stochastic~\acrshort*{gp} kernels.
 The combination of latent features and interpretable covariance structures brings a new tool to understand multiple time series better. Our model outputs human-readable reports with high-level abstraction as well as the relation among time series. We believe such results potentially facilitate the process of decision making in many fields i.e. scientific discovery, financial management. 

This paper offers the following contributions: (1) we introduce the~\gls*{lkm}, justify its well-definedness and develop its approximate inference algorithm; (2) we introduce a search procedure applicable to multiple time series and our working model; (3) an application making comparison reports among multiple time series.

This paper is structured as follows. Section~\ref{sec:models} presents our~\acrshort*{lkm}. Section~\ref{sec:structure_discovery} introduces a search procedure working with this model. Section~\ref{sec:exp} shows our experiments on several real-world data sets and gives comparison reports produced from our models. We conclude in Section~\ref{sec:conclusion}.
\section{Related work}
\label{sec:related_work}
In the compositional kernel, there have been efforts on improving the efficiency of model selection i.e. using Bayesian optimization, or sparse GP~\cite{BO_Model_NIPS16,KimT17,structured_ae} and relating human cognitive procedures~\cite{compositional_nips2016,SCHULZ201744}. Recently,~\cite{differential_ckl} proposed a neural network construction of compositional kernels with a guarantee in approximation capacity. Yet, the framework is less interpretable. For multiple time series,~\cite{Hwangb16} introduced a global \textit{shared} information among multiple sequences and individual kernels for each kernel. Our model is more general because no strong correlation assumption is required, the relation among time series is automatically discovered by~\acrshort*{ibp} matrix instead.

Stochastic grammar for ABCD~\cite{SchaechtleZRSM15} is introduced where interpretable kernels are selected via Bayesian learning over a binomial distribution imposed on the presence of kernels. It provides a sampling approach based on Venture probabilistic programming language~\cite{venture}. Another work~\cite{abcd_pp} represents kernel compositions in Stan language~\cite{stan}. A recent work~\cite{bayes_synth} built on the top of Venture as well presents a program synthesis approach to extract compositional kernels. 
However, these works only can apply to a single time series. While in our case, we work on multiple time series using IBP prior with an in-depth investigation of the model construction.

In the multi-task learning perspective, multi-task learning for GP regression has been studied extensively~\cite{semiparam,BonillaCW07,kernel_vector, gprn_wilson12icml,spike_lab,GuarnizoAO15_IBP_GP,guarnizo_indian_2015}. These methods commonly share limitations that GP kernel structures are fixed or given, not having the flexibility in selecting GP kernels. The additive kernel construction of our model is common with the Linear Model of Coregionalization (LMC)~\cite{kernel_vector} and extensions~\cite{sparse_convolved_gp,cross_spectral_gp,MOSM} where kernels are constructed by a linear combination of kernels.
While LMC optimizes these weights together with GP hyperparameters, our model is based on a Bayesian approach to infer $\vZ$. More importantly, the binary latent matrix $\vZ$ enhances the interpretability transparency over real-valued weights.

In terms of stochastic kernel generation,~\cite{levy_process_NIPS17} proposed a L\'evy kernel process where the mixture of kernels is obtained by placing a L\'evy prior over the corresponding spectral density. The~\acrshort*{lkm} is one of the attempts to put uncertainty on kernel constructions using \acrshort*{ibp} prior to select a set of interpretable kernels. 

It is worth mentioning methods which learn complex functions including convolutional networks~\cite{LeCun:1989} and sum-product networks~\cite{PoonD11}. AND-like and OR-like operation have the intuitively similar mechanisms of multiplication and summation in compositional kernels. Beyond this similarity between these operations and composing kernel operations, our work targets to study multiple complex functions where sharing kernels can be understood as AND-like operation among sequences.

\section{Background}

In this section, we provide a brief review of the \acrfull*{abcd} framework~\cite{grosse2012exploiting,DuvLloGroetal13,Lloyd2014ABCD,Ghahramani15_nature} and~\acrfull*{ibp}~\cite{IBP_GriffithsG05}.

\paragraph{\acrfull*{gp}} \acrfull*{gp}~\cite{Rasmussen_GPM} is defined as a multivariate Gaussian distribution over a (possibly infinite) collection of random variables. Whenever we select a subset from this collection, the distribution over the subset also is Gaussian. Commonly, GP is used as a prior over function values, denoted as $f(x) \sim \mathcal{GP}(m(x), k(x,x'))$ with $m(x)$ is the mean function, $k(x,x')$ is the covariance (kernel) function. In practice, the mean function is usually chosen as a zero mean function. Like many other kernel methods, kernel tricks are applicable to construct new kernels for~\acrshort*{gp}, be one of the key properties in the framework that we will describe next.

\paragraph{The \acrshort*{abcd} framework} The ABCD framework follows a typical Bayesian modeling process (see~\citet{inference_MacKay}), being composed of several parts e.g. a language of models, a search procedure among models, and a model evaluation. The framework makes use of~\acrfullpl*{gp} to perform various regression tasks.


Selecting kernel functions plays a crucial role in learning GP.~\acrshort*{abcd} searches a model out of an open-ended language of models which is constituted from a context-free grammar and base kernels. The base kernels model different characteristics of data such as white noise ($\textsc{WN}$), constant ($\textsc{C}$), smoothness ($\kSE$), periodicity ($\kPer$), and trending ($\kLin$) (see Appendix~\ref{appendix:base}). The grammar makes it possible to explore and generate new kernels from base ones via composition rules such as the product rule and the sum rule. A greedy search is applied in~\acrshort*{abcd} like in~\citet{grosse2012exploiting}, picking the most appropriate model based on a criterion e.g.~\gls*{bic}. Once the search procedure is finished, a human-readable report is generated from the interpretability of~\acrshort*{gp} base kernels and their compositions.
\paragraph{Indian Buffet Process} The \acrshort*{ibp}~\cite{IBP_GriffithsG05} defines a distribution over a binary matrix $\vZ$ with a finite number of rows and an infinite number of columns:
$\vZ \sim \textrm{IBP}(\alpha),$
with $\alpha$ is the concentration parameter. The matrix indicates feature assignments where the element at the $i$-th row  and the $j$-th column expresses the presence or absence of the $j$-th feature in the  $i$-th object. A natural application of~\acrshort*{ibp} is the linear-Gaussian latent feature model (LFM)~\cite{IBP_GriffithsG05}. Data represented by $\vX$ is factorized into an~\acrshort*{ibp} latent matrix $\vZ$ multiplying with a feature matrix $\vA$ with a Gaussian noise matrix $\mathcal{E}$:
$
\vX = \vZ\vA + \mathcal{E}.
$

\section{\acrfull*{lkm}}
\label{sec:models}

In this section, we define the~\acrfull*{lkm} and discuss its theoretical properties and unique characteristics. Then we will introduce inference algorithms for~\acrshort*{lkm}.

\subsection{Definition}
\paragraph{Notation} Let us denote $\vx_n =(x_{n1},..., x_{nD})^\top$ be a vector representing the $n$-th time series where $x_{nd}$ is the data point of the $n$-th time series at the $d$-th time step $t_d$. Here, $N$ is the number of time series and $D$ is the number of data points in each time series. To clarify further notations, we denote a data matrix $\vX$ taking $\vx_n, n = 1\dots N$ as rows. We introduce a latent matrix $\vZ$ taking $\vz_n, n=1\dots N$ as rows. 

\label{sec:lkm}
Given a set of~\acrshort*{gp} kernels $\{\vC_k\}_{k=1}^K$, we wish to model each time series $\vx_n$ with
\begin{equation}
\begin{split}
\vZ &\sim \operatorname{IBP} (\alpha),\\
\vf_{n} &\sim {\cal GP}(\mathbf{0}, \sum_{k=1}^K z_{nk} \vC_k), \\
\vx_{n}  &\sim {\cal N} (\vf_{n}, \sigma^2_n \vI),
\end{split}
\label{eq:lkm}
\end{equation}
where $\alpha$ is the~\acrshort*{ibp} concentration parameter. By the above model construction, an observation $x_{nd}$ corresponds to a~\acrshort*{gp} latent function variable $f_n(t_d)$. The $p(\vX|\vZ)$ is the product of all $p(\vx_n|\vz_n)$ where
\begin{equation}
	 p(\vx_n|\vz_n) = \left|2\pi \vD(\vz_n)\right|^{-\nicefrac{1}{2}} \exp\left(-\frac{1}{2}\vx^\top_n \vD(\vz_n)^{-1}\vx_n  \right),
	 \label{eq:likelihood_n}
\end{equation}
with $\vD(\vz_n) = \sum_{k=1}^K z_{nk}\vC_k + \sigma_n^2 \vI $, and $z_{nk} \in \{0,1\}$ is the element of $N\times K$ matrix $\vZ$ indicating whether the $n$-th time series has additive kernel $\vC_k$. Since we place~\acrshort*{ibp} on $\vZ$, it can have infinitely many columns as $K \rightarrow \infty$. This model focuses on the process of creating the stochastic kernel $\vD(\vz_n)$ for each $\vx_n$. The kernel selection procedure relies on learning~\acrshort*{ibp} matrix via Bayesian inference. 

\subsection{Properties}
\paragraph{Well-definedness of \acrshort*{lkm}}
Since an IBP prior is imposed on the matrix $\vZ$, the number of its columns can go to infinity. Thus we may have an infinite number of kernels. It is important to verify whether $p(\vX|\vZ)$ forms a well-defined probability distribution even with an infinite number of kernels.~\citet{IBP11} gave a detailed analysis in the case of LFM. In fact, $p(\vX|\vZ)$ in LFM is independent to feature matrix because of marginalization over feature matrix. However, $p(\vX|\vZ)$ in~\acrshort*{lkm} is still associated with kernels in its representation. We will justify the well-definedness in the case of~\acrshort*{lkm} as follow.
\begin{prop}
	The likelihood of~\acrshort*{lkm} is well-defined.
\end{prop}
\begin{proof}
	
	The likelihood can be easily obtained by
	$$p(\vX|\vZ) = \prod_{n=1}^N p(\vx_n | \vz_{n}).$$
	We will use \emph{lof} operator on $\vZ$. The~\emph{lof} transforms a binary matrix by reordering its columns by the binary number associated to that column~\cite{IBP11}. Since all kernels $\vC_k$ are commutative, \emph{lof} performs on $\vZ$ without affecting $p(\vX|\vZ)$ as kernels are exchanged accordingly.\\
	We apply \emph{lof} on $\vZ$ to obtain $[\vZ^{+} \vZ^0]$ where $\vZ^{+}$ contains $K^+$ nonzero columns and $\vZ_{0}$ contains $K^0$ zero columns. Each row in $\vZ_{+}$ contributes to generate kernel $\vD(\vz_n) = \sum_{k=1}^{K^+}z_{nk}^+\vC_k + \sigma_n^2\vI$. When $K \rightarrow \infty$, $K^+$ still stays finite as the property of IBP. Thus, $\vD(\vz_n)$ is now the sum of a finite number of covariances kernels $\vC_k$. This means that each multivariate Gaussian likelihood $p(\vx_n | \vz_{n})$ has a well-defined covariance. Finally, we can conclude that $p(\vX|\vZ)$ is well-defined.
\end{proof}
With the above proposition, IBP prior becomes a regularizer preventing the degradation of kernel construction (an explosion of the kernel variances) when increasing the number of kernels $K$.

\paragraph{Comparisons with existing models}
Feature sharing models~\cite{spike_lab,gprn_wilson12icml,GuarnizoAO15_IBP_GP} commonly represent data as
$$\vx_n = \sum_{k=1}^Kw_k\vf_k + \boldsymbol{\epsilon}_n,$$
with $\vf_k, k=1\dots K$ are shared features, $\boldsymbol{\epsilon}_n, n=1\dots N$ are Gaussian noise vectors. Each $\vf_k$ is a drawn GP realization from $\vC_k$. The $w_k$ can be placed spike and slab prior~\cite{spike_lab} or are samples from GPs~\cite{gprn_wilson12icml}. 

\begin{figure}[t]
	\centering
	\includegraphics[width=0.5\columnwidth]{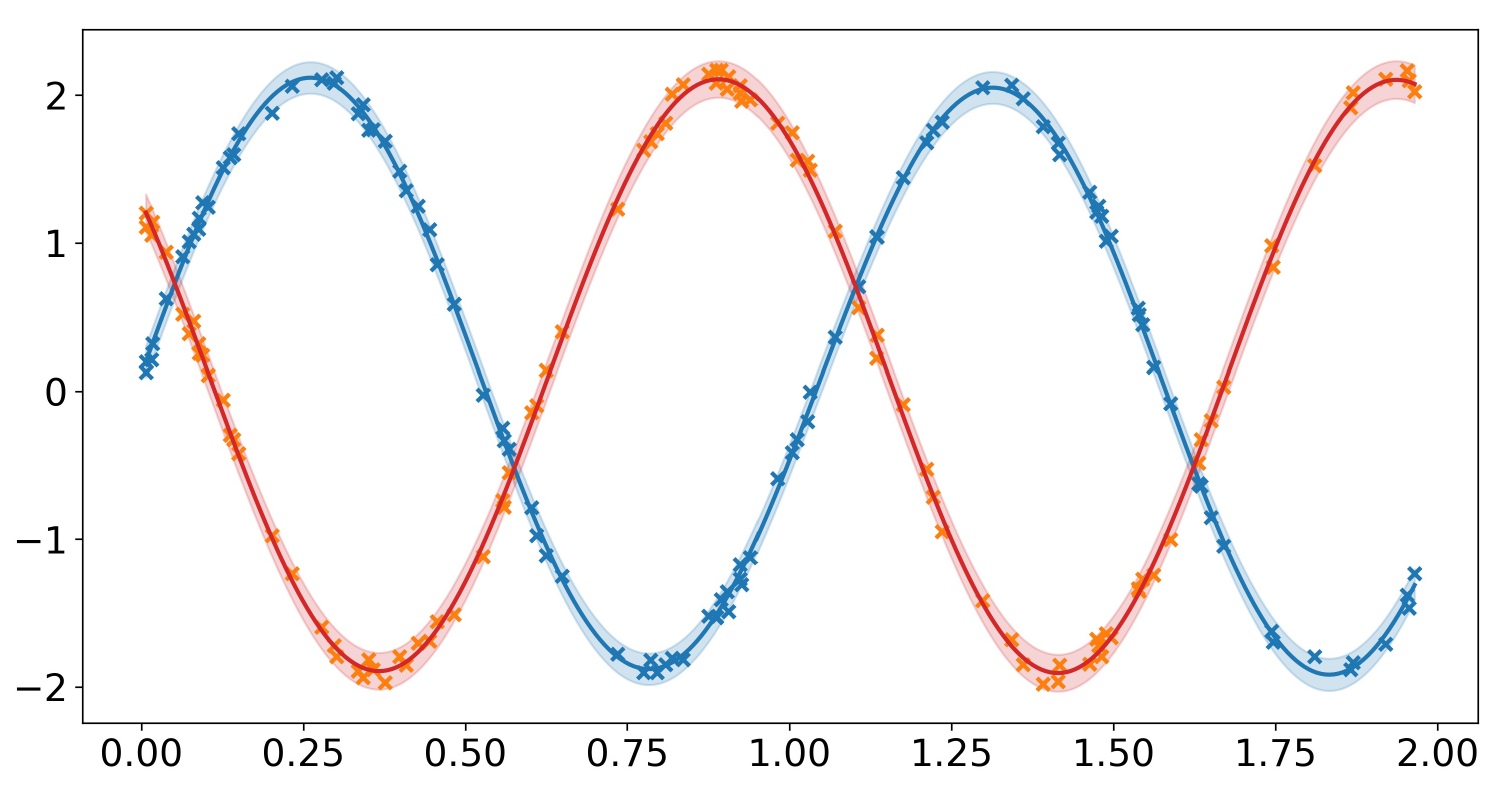}
	\captionof{figure}{Fitting two functions using LKM. The toy data set contains two realizations generated from a~\acrshort*{gp} prior with a periodic kernel.}
	\label{fig:realization} 
\end{figure}

Our~\acrshort*{lkm} is more expressive than the feature sharing family in terms of function realizations. Suppose the posterior decomposition of additive Gaussian distributions presents as: If $\vf = \vf_1 + \vf_2$, where $\vf_1 \sim {\cal N}(\mathbf{0}, \vK_1)$, $\vf_2 \sim {\cal N}(\mathbf{0}, \vK_2)$, the conditional distribution of $\vf_1$ given the sum $\vf$ is 
$$\textstyle \vf_1 | \vf \sim {\cal N}(\vK_1^\top(\vK_1 + \vK_2)^{-1} \vf, \vK_1 - \vK_1^\top(\vK_1 + \vK_2)^{-1}\vK_1).$$
In the multiple time series setting, each decomposed component under the same~\acrshort*{gp} prior could be realized differently in different time series. In other words, for a specific $k$, the posterior $\vf_{k} | \vx_{n}$ varies whenever $\vx_n$ changes even with the fixed covariance $\vC_k$. A simple setup in Figure~\ref{fig:realization} can verify this observation. We generate two sequences from a single periodic GP and then run~\acrshort*{lkm} on this data with two different periodic kernels $\vC_1$ and $\vC_2$. When we learn~\acrshort*{lkm}, $\vZ = [0, 1; 0, 1]$ is obtained. That is,~\acrshort*{lkm} is able to recognize these two realizations from one~\acrshort*{gp}. 

We also emphasize that the Bayesian approach that is considered in our kernel construction, can be viewed as a stochastic kernel generative process~\cite{levy_process_NIPS17}.

Figure~\ref{fig:graphical_model} illustrates the plate notations of ~\acrshort*{lkm} and R-ABCD~\cite{Hwangb16}. R-ABCD shares a global kernel for all time series and allocates a distinctive kernel $\vC_n$ for each time series. Note that spectral mixture kernel~\cite{WilsonA13} is used for $\vC_n$ in R-ABCD prevents ones from deriving interpretable models.

\begin{figure}[t]
	\centering
	{\includegraphics[width=0.27\textwidth]{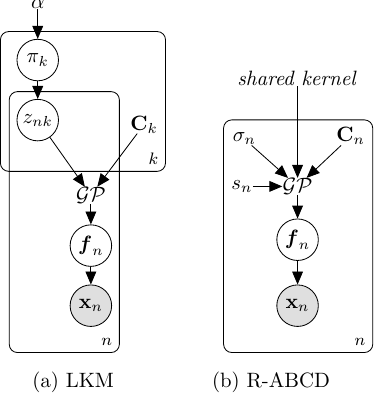}}
	\captionof{figure}{Graphical model of (a)~\acrshort*{lkm} and (b) R-ABCD.}
	\label{fig:graphical_model}
\end{figure}

\subsection{Inference algorithm}
\label{sec:vi}
\paragraph{\Acrlong*{vi}} Variational inference methods approximate the true posterior $p(\vZ|\vX)$ by a variational distribution $q(\vZ)$. The method converts the optimization problem of KL divergence between $p$ and $q$ into an equivalent problem by maximizing the evidence lower bound (ELBO) $\cal L$, 
\newcommand{\bpi}{\boldsymbol{\pi}}
\begin{align*}
\log   p(\vX) 
&\geq \mathbb{E}[\log p(\vX, \vZ)] + H[q] \\
&= \expect[\log p(\vZ)] + \expect[\log p(\vX|\vZ)] + H[q]  \triangleq \mathcal{L} .
\end{align*}
where $\mathbb{E}$ indicates the expectation over the approximate posterior distribution $q(\vZ)$, and $H[q]$ is the entropy of $q$. The last equation in the above derivation comes from the model definition in Equation~\ref{eq:lkm} where the joint distribution $p(\vX, \vZ)$ is in the form of $p(\vX|\vZ)p(\vZ)$. Here, we choose the variational distribution $q(\vZ)$ in the mean-field family. It is factorized into $q(z_{nk}) = \operatorname{Bernoulli}(z_{nk};\nu_{nk})$. 

The first term $\expect[\log p(\vZ)]$ in $\mathcal{L}$ is explained in Appendix~\ref{appendix:expect}~\cite{vibp_Doshi}.

Now our main focus is to estimate $\expect[\log p(\vX|\vZ)]$. Recall that ${p(\vX|\vZ)} {=} {\prod_{n=1}^N p(\vx_n |\vz_n)}$, we can break $\expect[\log p(\vX|\vZ)]$ into the sum of $\mathbb{E}[\log p(\vx_{n} | \vz_{n})]$. The evaluation of each $\mathbb{E}[\log p(\vx_{n} | \vz_{n})]$ is expensive since it needs to compute the expectation of GP likelihood functions associated with discrete random variables $\vZ$. Specifically, $\mathbb{E}[\log p(\vx_{n} | \vz_{n})]$ is written as the sum of $-\frac{1}{2} \vx_{n}^\top \mathbb{E}\left[\vD(\vz_n)^{-1} \right] \vx_{n}$ (or the expectation of data-fit term in \acrshort*{gp} likelihood), $- \frac{1}{2} \mathbb{E} \left[\log \left|2\pi \vD(\vz_n) \right| \right]$ (or the expectation of \acrshort*{gp} model complexity) and a constant term.
Each expectation is the sum of following $2^K$ terms: (1) $p(\vz_{n} = \vt) \vD(\vt)^{-1}$ for all $\vt \in \{0,1\}^K $ in the case of the expectation of inverse matrix; (2) $p(\vz_{n} = \vt) \log \left|2\pi \vD(\vt) \right| $ for all $\vt \in \{0,1\}^K $ in the case of the expectation of log-determinant. Hence, it is not practical to estimate an exponential number of inverse and determinant operations. 

\paragraph{Relaxation}To mitigate the difficulty in estimating $\expect [\log p(\vx_n|\vz_n)]$, we first relax the discrete random variables $z_{nk}$ to a continuous ones, then estimate the expectation using Monte Carlo method. The relaxation turns the Bernoulli random variables $z_{nk} \sim \textrm{Bernoulli}(\nu_{nk})$ into 2-dimensional continuous random variable $[\tilde{z}_{nk}, \underaccent{\tilde}{z}_{nk}] \sim \textrm{Concrete}(\nu_{nk}, \lambda)$, where $\lambda$ is the temperature parameter~\cite{gumbel1}. Here, the categorical random variable $[z_{nk}, 1 - z_{nk}]$ corresponds to the relaxed one $[\tilde{z}_{nk}, \underaccent{\tilde}{z}_{nk}]$. We are interested in $\tilde{z}_{nk}$ which corresponds to $z_{nk}$. A sample of $\tilde{z}_{nk}$ is drawn by sampling $g_1$ and $g_2$ from $\textrm{Gumbel}(0,1)$ and computing as
$$\tilde{z}_{nk} = \frac{\exp(\frac{\log(\nu_{nk}) + g_1}{\lambda})}{\exp(\frac{\log(\nu_{nk}) + g_1}{\lambda})+ \exp(\frac{\log(1-\nu_{nk}) + g_2}{\lambda})}.$$
This is known as the Gumbel-Softmax trick~\cite{gumbel1,gumbel2}. The unbiased estimation of $\expect [\log p(\vx_n|\vz_n)]$ after relaxation is 
$$
\expect[\log(p(\vx_n|\vz_n)] \approx \frac{1}{m}\sum_{i=1}^m \log p(\vx_{n}| \tilde{\vz}_n^{(i)})),
$$
where $m$ is the number of samples, $\{\vz_n^{(i)}\}_{i=1}^m$ is the set of samples. The kernel $\vD(\tilde{\vz}_n)$ now takes all $\vC_k$ into account since $\tilde{\vz}_n$ is in $(0,1)^K$ instead of $\{0,1\}^K$.
Now the number of evaluations on matrix inversions and determinants is the number of sample $M$, instead of the number of all (exponential) configurations generated from $K$ binary random variables $\vz_n$. Moreover, the estimation benefits from this reparameterization trick to estimate gradients in stochastic computation graph~\cite{SCG}.

\section{Structure discovery in multiple time series}
\label{sec:structure_discovery}

In this section, we present a search algorithm to discover~\acrshort*{gp} compositional kernels for multiple time series.
\paragraph{Search scheme} To cope with the broad structure space, our algorithms follows the principle of greedy algorithms~\cite{grosse2012exploiting,DuvLloGroetal13,Lloyd2014ABCD}. That is, we maintain a set of additive kernel structures $\{\mathcal{S}_d^{(k)}| \mathcal{S}_d^{(k)} = \prod_l \mathcal{B}^{(k_l)}_d \text{ with } \mathcal{B}^{(k_l)}_d \text{s are base kernels}, k=1\dots K\}$ at a search depth $d$. We map correspondingly $\mathcal{S}_d^{(k)}$ to the required kernels $\vC_k$ in~\acrshort*{lkm}. At the next depth, the set will recruit new additive kernels by expanding some of the elements of the set at the current depth $d$. The context-free grammar rules of the expansion are the same with~\acrfull*{ckl}~\cite{DuvLloGroetal13}. However, for the case when $\mathcal{S}_d^{(k)}$ is expanded into a new kernel which is written in an additive form as $\sum_{m=1}^M \mathcal{S}_{d+1}^{(k_m)}$, we will consider this expansion as $M$ separated expansions $\mathcal{S}_d^{(k)} \rightarrow \mathcal{S}_{d+1}^{(k_m)}$. The generated structures $\mathcal{S}_{d+1}^{(k_m)}$ are added to the set rather than the sum $\sum_{m=1}^M \mathcal{S}_{d+1}^{(k_m)}$. This procedure always makes new candidate structures satisfy the definition of $\{\mathcal{S}_d^{(k)}\}$ without assuming an arbitrary sum.
\paragraph{\Gls*{pe}} Our search algorithm iteratively expands $\mathcal{S}_d^{(k)}$ and obtain a set of candidates $\{\mathcal{S}_d^{(k_1)}, \dots, \mathcal{S}_d^{(k_m)} \}$. We make a new set which is the union of the previous one excluded the selected structure $\{\mathcal{S}_d^{(k)} \}_{k=1}^K \backslash \{\mathcal{S}_d^{(i)}\}$ and the new candidate structures $\{\mathcal{S}_d^{(i_1)}, \dots, \mathcal{S}_d^{(i_m)} \}$ (Figure~\ref{fig:expansion}).
Our~\acrlong*{vi} algorithm (described in Section~\ref{sec:vi}) learns $\vZ$ and \acrshort*{gp} kernels. If there is an improvement in~\acrshort*{bic}~\cite{BIC}, we keep the updated kernel set. Otherwise, it rolls back to the previous one. We proceed to the next expansion using this updated one (Algorithm~\ref{alg:pe}).

Advantages of our~\acrshort*{pe} algorithm are (1) it does not make drastic increases in structure space in each expansion, (2) it carefully assesses models by a selection criterion (\acrshort*{bic}) and flexibly falls back to the previous model if the criterion does not select the new one, (3) the fewer number of kernels in~\acrshort*{pe} makes us easier to initialize~\acrshort*{gp} hyperparameters as well as reduce the number of restarts learning $\vZ$.
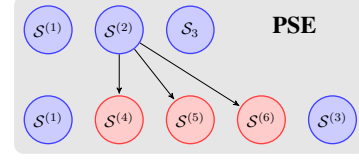
\begin{figure}
	\centering
    \scalebox{0.9}{
		\scalebox{0.7}{
\begin{tikzpicture}[node distance=1.5cm,>=stealth',bend angle=45,auto]

  \tikzstyle{place}=[circle,thick,draw=blue!75,fill=blue!20,minimum size=10mm]
  \tikzstyle{red place}=[place,draw=red!75,fill=red!20]
  \tikzstyle{yellow place}=[place,draw=yellow!75,fill=yellow!20]
  \tikzstyle{green place}=[place,draw=green!75,fill=green!20]
  \tikzstyle{transition}=[rectangle,thick,draw=black!75,
  			  fill=black!20,minimum size=10mm]
  
  \tikzstyle{box}=[rectangle, rounded corners, minimum height=2cm, minimum width=5cm, align=center,fill=#1,draw=gray!70,line width=1pt]

  \tikzstyle{every label}=[red]

  \begin{scope}
    \node [place]  (s11) {$\mathcal{S}^{(1)}$};
    
    \node [place]	(s21) [below of=s11, yshift=-4mm] {$\mathcal{S}^{(1)}$};
    \node [red place]	(s22) [right of=s21] {$\mathcal{S}^{(4)}$};
    \node [red place]	(s23) [right of=s22] {$\mathcal{S}^{(5)}$};
    \node [red place]	(s24) [right of=s23] {$\mathcal{S}^{(6)}$};
    \node [place]	(s25) [right of=s24] {$\mathcal{S}^{(3)}$};
    \node [place]	(s12) [right of=s11]{$\mathcal{S}^{(2)}$}
    	edge [post] (s22)
        edge [post] (s23)
        edge [post] (s24)
    ;
    \node [place]	(s13) [right of=s12]{$\mathcal{S}_3$};
  \end{scope}
  
	\node[above of=s25, , xshift=-0.8cm, yshift=5mm] {\Large \textbf{PSE}};

  \begin{pgfonlayer}{background}
    \filldraw [line width=4mm,join=round,black!10]
      (s11.north  -| s25.east)  rectangle (s25.south  -| s21.west);
  \end{pgfonlayer}
  
\end{tikzpicture}
}
    }
	\caption{\acrshort*{pe} with $\mathcal{S}^{(2)}$ expanded into 3 others to create a new set. }
	\label{fig:expansion}
\end{figure}

\begin{algorithm}[tb]
	\caption{Partial set expansion of~\acrshort*{lkm} learning}
	\label{alg:pe}
	\begin{algorithmic}[1]
		\REQUIRE Input data and  search depth $D$, initial $\{\mathcal{S}_d^{(k)}\}$
		\FOR {$d=1\dots D$}
			\FOR{ $\mathcal{S}$ in $\{\mathcal{S}_d^{(k)}\}$ of depth $d$}
				\STATE Update $\{\mathcal{S}_d^{(k)}\} \leftarrow \{\mathcal{S}_d^{(k)}\} \backslash \mathcal{S} \cup expand(\mathcal{S})$
				\STATE Run~\acrshort*{lkm} learning
				\IF {improvement in BIC}
					\STATE Use this updated set $\{\mathcal{S}_d^{(k)}\}$
				\ELSE
					\STATE Rollback to previous set $\{\mathcal{S}_d^{(k)}\}$
				\ENDIF
			\ENDFOR
		\ENDFOR
	\end{algorithmic}
\end{algorithm}


Our kernel search procedure is a meta search algorithm inspired from oracle machines in computational theory~\cite{computational_complexity}. The LKM plays a role as an \emph{oracle}. Given a set of kernel structures, one tries to ask the oracle to decide the appropriate structures. The oracle will response an answer as $\vZ$ in our case. Exploiting the returned $\vZ$, the kernel structures will be elaborated more by performing~\acrshort*{pe}. The procedure is repeated by making new inquiry based on the expanded structures.

We emphasize that~\acrshort*{pe} with~\acrshort*{lkm} considers a larger number of kernel structures than those in~\acrshort*{ckl}. 
  Suppose that~\acrshort*{ckl} and our search algorithm have the same found structure at a depth $d$. Whereas the~\acrshort*{ckl}'s structure is $\mathcal{S}_d = \mathcal{S}_d^{(1)}  + \dots + \mathcal{S}_d^{(K)}$,~\acrshort*{pe} represents it as a set $\{\mathcal{S}_d^{(1)}, \dots,\mathcal{S}_d^{(K)}\}$. Let $L$ be the largest number of base kernels in $\mathcal{S}_d^{(k)}$, and $R$ be the maximum number of grammar rules per substructure. All possible search candidates in~\acrshort*{ckl} is $O(RK2^L + R2^K)$ kernels, while~\acrshort*{pe} incorporating with~\acrshort*{lkm} considers $O(K2^{R2^L + K})$ number of kernels. Detailed analysis is provided in Appendix~\ref{sec:compare_pe_ckl}.
  
Although our search algorithm explores a much larger search space than CKL in theory, the prior over $\vZ$ still limits the expressiveness power of our model. Moreover, learning $\vZ$ relies on a gradient-based method where the global optimal is not guaranteed. Thus, our kernel search algorithm may not find the optimal kernel over all the possible candidates.

\section{Experimental evaluations}
\label{sec:exp}
In this section, we describe data sets and demonstrate both qualitative and quantitative results.

\begin{figure*}[h]
	\centering
	\scalebox{0.8}{\input{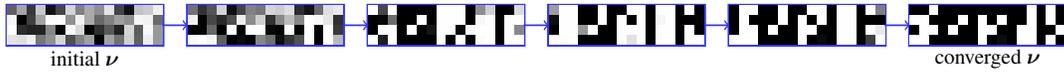}}
	\caption{\small The visualization of $\boldsymbol{\nu}$ as the training of~\acrshort*{lkm} goes on. The columns indicates time series. The row indicates kernels $\vC_k$.}
	\label{fig:learn_z}
\end{figure*}
 \begin{figure*}
 	\centering
 	\begin{tikzpicture}
 	
 	\node[inner sep=0pt] (2) at (0,-0.4) {\includegraphics[width=0.07\textwidth]{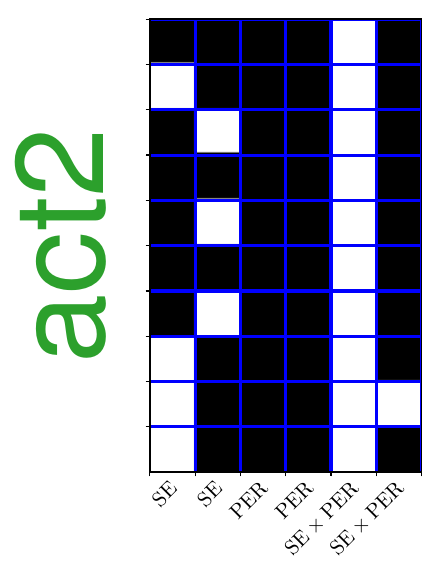}  };
 	\node[inner sep=0pt] (3) at (0,-2.) {\includegraphics[width=0.07\textwidth]{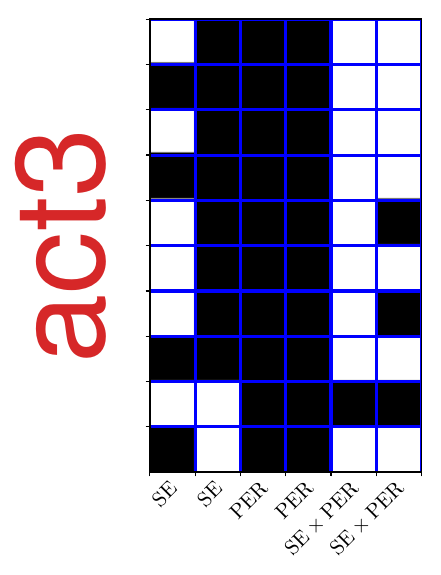}  };
 	\node[inner sep=0pt] (4) at (1.3,-0.4) {\includegraphics[width=0.07\textwidth]{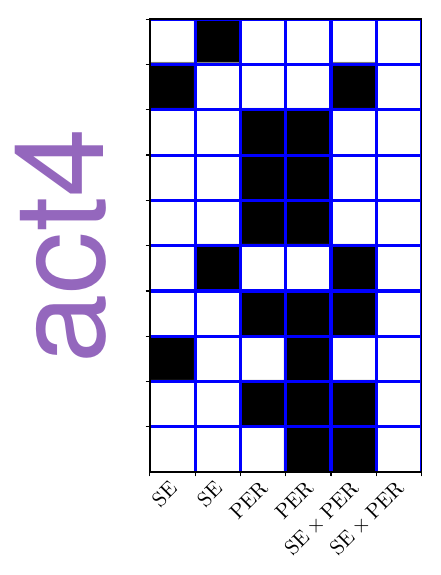}  };
 	\node[inner sep=0pt] (5) at (1.3,-2.) {\includegraphics[width=0.07\textwidth]{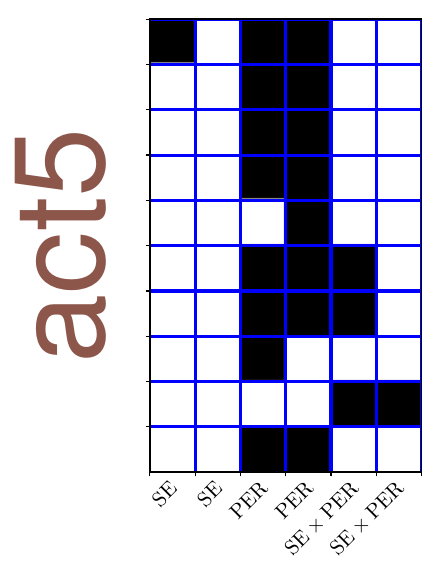}  };
 	\node[inner sep=0pt] (fit4) at (5.,-1) {\includegraphics[width=0.33\textwidth]{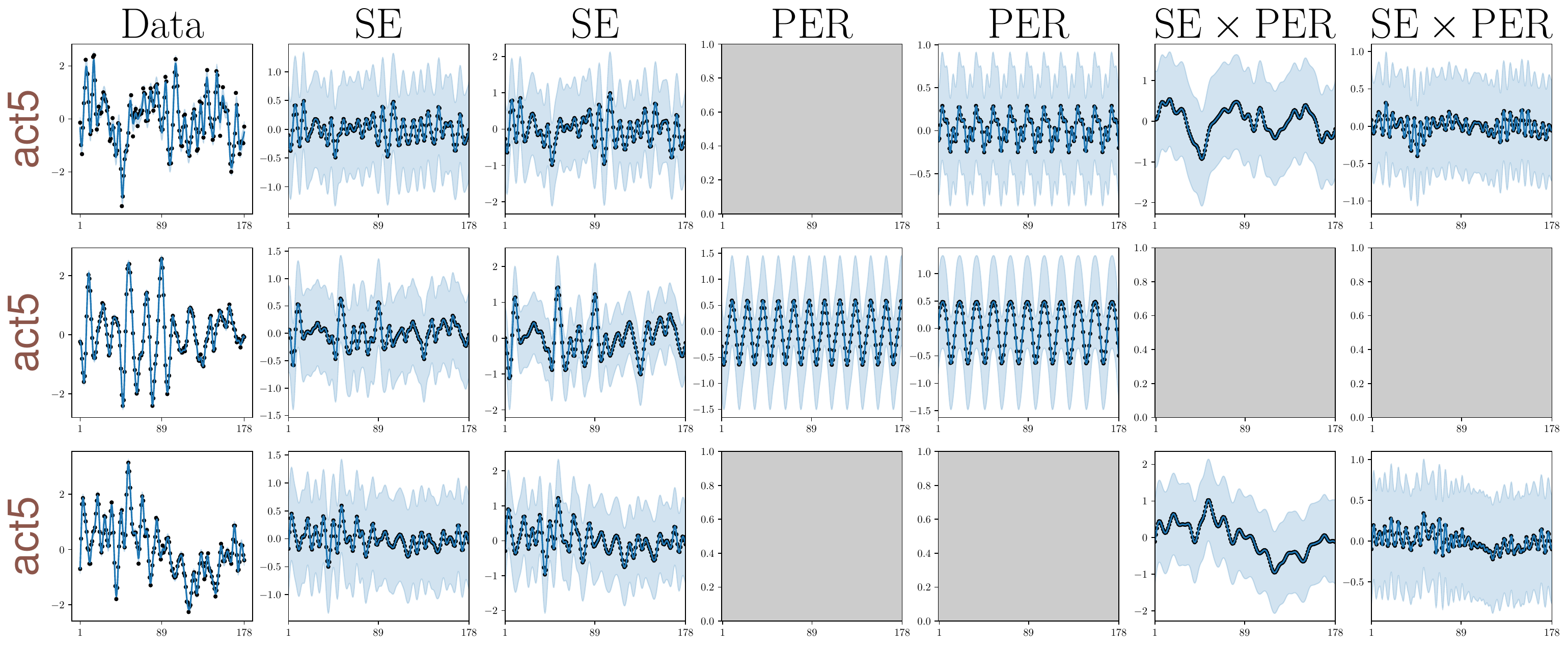}  };
 	
 	\node[] (a) at (3.3,-2.7) {(a)};
 	
 	\node[inner sep=0pt] (1) at (9.5,-1.2) {\includegraphics[width=0.07\textwidth]{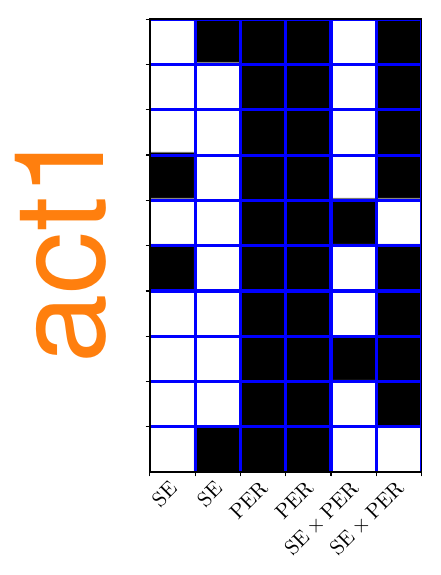}  };
 	\node[inner sep=0pt] (fit4) at (13.,-1) {\includegraphics[width=0.33\textwidth]{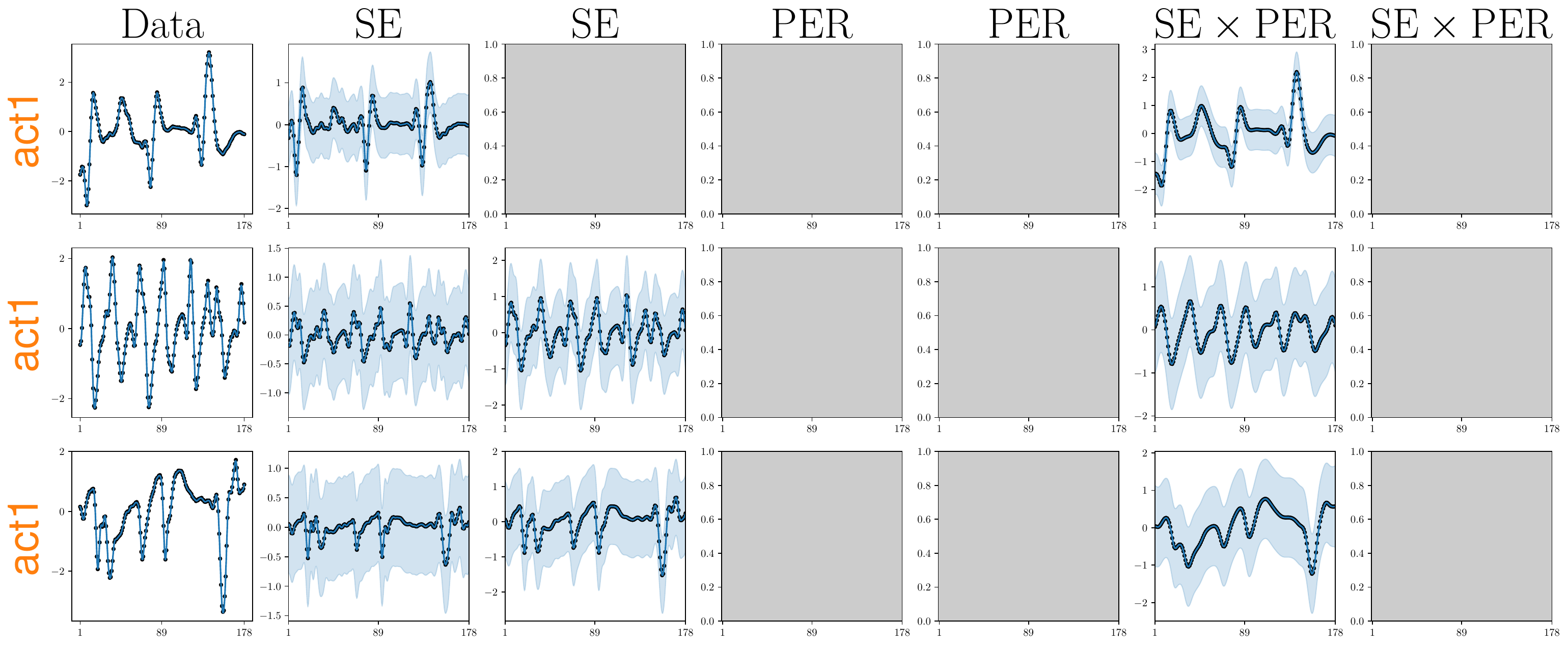}  };
 	
 	\node[] (b) at (11.3,-2.7) {(b)};
 	\end{tikzpicture}
 	\caption{\small Epileptic seizure data set. There are 5 activities of EEG recording: seizure (\textcolor{c1}{act1}), located tumor (\textcolor{c2}{act2}), identifying tumor (\textcolor{c3}{act3}), eyes closed (\textcolor{c4}{act4}), eyes open (\textcolor{c5}{act5}). (a) Non-seizure. \emph{Left}: part of learned $\vZ$ corresponding to each activity, black means $z_{nk}=0$, otherwise white; \emph{Right}: posterior of 3 last time series from \textcolor{c5}{act5} with their decomposition. (b) Seizure. \emph{Left}: part of learned $\vZ$ from \textcolor{c1}{act1}; \emph{Right}: posterior plot of 3 first time series from \textcolor{c1}{act1} with their decomposition. The missing subplots or gray background plots indicate $z_{nk}=0$. } 
 	\label{fig:seizure}
 \end{figure*}


\subsection{Real-world time series data} 
\textbf{Strongly correlated data sets}\hspace{5pt}  We tested our algorithm on three different data sets: US stock prices, US housing markets and currency exchanges. These data sets are well-described and publicly accessible~\cite{Hwangb16}. The US stock price data set consists of 9 stocks (GE, MSFT, XOM, PFE, C, WMT, INTC, BP, and AIG) containing 129 adjusted closes taken from the second half of 2001. The US housing market data set includes the 120-month  housing prices of 6 cities (New York, Los Angeles, Chicago, Phoenix, San Diego, San Francisco) from 2004 to 2013. The currency data set includes 4 currency exchange rates from US dollar to 4 emerging markets: South African Rand (ZAR), Indonesian Rupiah (IDR), Malaysian Ringgit (MYR), and Russian Rouble (RUB). Each currency exchange time series has 132 data points.  \\
\textbf{Heterogeneous data set}\hspace{5pt} We collected time series from various domains into a data set. It consists of gold prices, crude oil prices, NASDAQ composite index, and USD index\footnote{Quandl codes respectively are \texttt{WGC/GOLD\_DAILY\_USD}, \texttt{FRED/DCOILBRENTEU}, \texttt{NASDAQOMX/COMP},\texttt{FRED/DTWEXM}} from 2015 July 1$^{\text{st}}$ to 2018 July 1$^{\text{st}}$. We call this data set as GONU (\underline{G}old, \underline{O}il, \underline{N}ASDAQ, \underline{U}SD index). Each time series has 157 weekly prices or indexes taken from~\citet{quandl}. The interactions between this sets of time series are known to be complex. For instance, the gold and oil prices might have a negative correlation where one may increase but the other decreases. There are many studies in the financial research focusing on these target time series~\cite{financial_1,financial_2}.

\textbf{Epileptic seizure data set} We retrieved the epileptic seizure data set~\cite{seizure} from UCI repository~\cite{UCI_repo}. This data set contains EEG recordings of brain activities for 23.6s. Each record corresponds to one out of 5 activities including eyes open, eyes closed, identifying the tumor, located the tumor and seizure activity. Each time series contains 178 data points.

\subsection{Qualitative results}
With the motivation that interpretable machine learning models can help understand data better, thereby fostering scientific discovery and decision making, we carried experiments on the mentioned data sets to demonstrate the potential applicability of our search algorithm on LKM. 

 \subsubsection{Exploiting information from $\vZ$}
 \textbf{Learning $\vZ$}\hspace{5pt}  We visualize the variational parameters $\boldsymbol{\nu}$ in Figure~\ref{fig:learn_z}. The value of ${\nu_{nk}}$ is the probability of $z_{nk}=1$. The bigger $\nu_{nk}$ is, the more probable the kernel $\vC_k$ is selected for time series $\vx_n$. 
 
 \textbf{Interpreting $\vZ$} We randomly take 50 time series from the epileptic seizure data where each activity has 10 time series. Because finding a covariance kernel decomposition for a large number of time series is time-consuming, and therefore prohibits kernel structure search, we looked for latent kernels from the set of kernels $\{\kSE_1,\kSE_2, \kPer_1, \kPer_2, \kSE_3 \times \kPer_3, \kSE_4 \times \kPer_4\}$. Figure~\ref{fig:seizure} illustrates a summary of the model outputs. Readers may refer Appendix~\ref{sec:full_seizure} for the full output.
 
 We observe several interesting properties. Located tumor and identifying tumor are quite similar because the corresponding block matrix from $\vZ$ has the same sparsity. Also, having fewer active $\kSE$ kernels indicates that they do not vary much. The activities of opening eyes and closed eyes commonly have rapidly varying signals with small lengthscales. The seizure, on the other hand, has a similar level of sparsity comparing to those of opening eyes or closed eyes. However, there is no sign of low-frequency periodic pattern.
 
The latent matrix $\vZ$ encodes certain relations between time series in the light of kernel interpretability. Next, we fully employ the description of kernels to generate comparison reports.

\subsubsection{Comparison report}
 \begin{figure}[t!]
 	\centering
 	\scalebox{1.55}{
 		\input{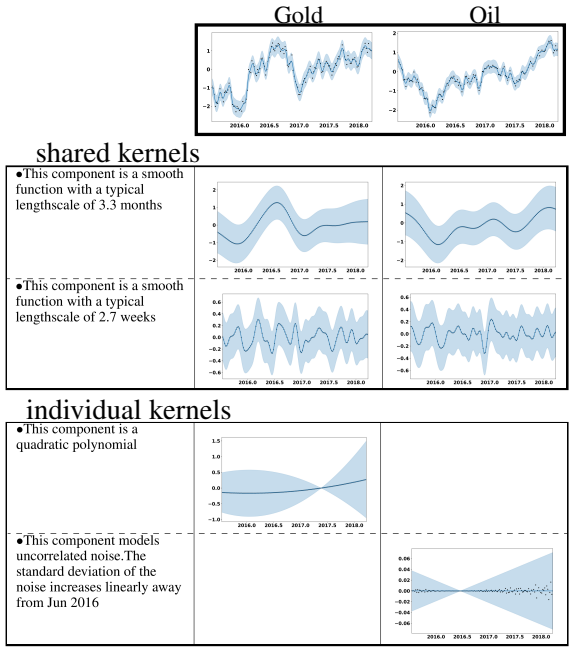}
 	}
 	\caption{\small A part of pairwise comparison between Gold and Oil in GONU data set. The uppermost plots are the posterior means and variances of two time series.
 		The remaining plots contain shared components and individual components with descriptions and posteriors $\vf_k|\vx_n$ for each time series. The blank in the individual components means ``not available".}
 	\label{fig:pair}
 \end{figure}

\textbf{Overview comparison} \hspace{5pt}By taking the advantage of the learned latent matrix $\vZ$ and the descriptive properties of found GP covariance structures, we generate a human-readable report containing the comparison among time series. For example, the generated text can have formats like 
\begin{center}
``[$T_1,\dots, T_m$] share [\emph{description}]''
\end{center}
where the replacement of [$T_1,\dots, T_m$] is a set of time series, [\textit{description}] is generated by the found GP structure. Below is extracted from GONU data set.
\noindent\fbox{%
	\parbox{0.46\textwidth}
	{%
		\textbullet \hspace{6pt} Gold, Oil, NASDAQ, USD index share the following property:\\
		This component is periodic with a period of 1.4 years but with varying amplitude. The amplitude of the function increases linearly away from Apr 2017. The shape of this function within each period has a typical lengthscale of 4.9 days.\\
		\textbullet \hspace{6pt}	Gold, Oil, USD index share the following property:\\
		This component is a smooth function with a typical lengthscale of 2.7 weeks.\\
		\textbullet \hspace{6pt} NASDAQ has the following property:\\
		This component is a linear function.
	}
}
\begin{table*}[t!]
\centering
	\scalebox{0.94}{
		\begin{tabular}{@{}c|cc|cc|cc|cc@{}}
			&  \multicolumn{2}{c|}{9 stocks} &  \multicolumn{2}{c|}{6 houses} & \multicolumn{2}{c|}{4 currencies} & \multicolumn{2}{c}{GONU} \\ 
		                   &  RMSE    & MNLP & RMSE     &  MNLP & RMSE        & MNLP  & RMSE &MNLP \\\midrule 
			Spike and Slab &  $10.07_{\pm0.12}$   & $2.87_{\pm0.05}$& $10.85_{\pm0.46}$    &  $6.92_{\pm0.09}$ & $174.71_{\pm14.52}$      & $4.09_{\pm0.10}$  & $1.07_{\pm0.08}$ & $2.36_{\pm0.11}$ \\
			GPRN           & $6.11_{\pm0.09} $    & $2.78_{\pm0.14}$ & $8.96_{\pm0.17}$     &  $6.64_{\pm0.46}$ & $193.13_{\pm49.40}$ & $4.24_{\pm0.20}$  & $1.16_{\pm0.12}$ & $2.46_{\pm0.28}$ \\
			LMC     & $8.20_{\pm0.53} $  & $2.24_{\pm0.23}$ &  $11.31_{\pm1.04}$ & $5.90_{\pm0.46}$     & $394.83_{\pm40.54}$ & $4.90_{\pm0.15}$  & $1.01_{\pm0.14}$ & $\mathbf{1.43_{\pm0.11}}$\\
			MOSM     & $5.48_{\pm1.01}$   & $2.97_{\pm0.01} $      &   $8.15_{\pm1.51}$ &  $5.90_{\pm0.20}$     & $318.26_{\pm101.52}$ &  $3.93_{\pm0.15}$  &  $0.84_{\pm0.18}$ & $3.13_{\pm1.06}$\\
			\acrshort*{abcd}           & $8.37_{\pm 0.03}$ & $2.58_{\pm 0.05}$  & $7.98_{\pm0.03}$  & $5.61_{\pm0.05}$  & $325.58_{\pm8.64}$ & $4.47_{\pm0.04}$   & $0.86_{\pm0.01}$& $2.21_{\pm0.03}$ \\
			\acrshort*{rabcd}          & $4.88_{\pm0.03}$  & $1.95_{\pm0.05}$  & $\mathbf{3.17_{\pm0.10}}$  & $6.07_{\pm0.09}$   & $208.32_{\pm5.02}$ &$\mathbf{3.62_{\pm0.03}}$  &  $0.97_{\pm0.03}$& $2.01_{\pm0.10}$
      \\
			\midrule
			\acrshort*{lkm}    & $\mathbf{4.58_{\pm0.16}}$  & $\mathbf{1.87_{\pm0.10}}$   & $4.37_{\pm0.16}$
			 &$\mathbf{5.54_{\pm0.40}}$& $\mathbf{133.00_{\pm16.92}}$ &$\mathbf{3.61_{\pm0.16}}$  & $\mathbf{0.76_{\pm0.07}}$ & $1.90_{\pm0.25}$     
			
		\end{tabular}
	}
	\captionof{table}{\small RMSEs and NMLPs for each data set with corresponding methods (5 independent runs per method). In most cases,~\acrshort*{lkm} has lower RMSEs and NMLPs compared to those of existing methods. }
	\label{tab:rmse}
\end{table*}

\begin{figure}
	\centering
	\scalebox{0.65}{
\begin{tikzpicture}

\begin{groupplot}[
	group style={
    	group size= 4 by 1, 
        horizontal sep=0.6cm,
        xticklabels at=edge bottom, 
        group name = bar_plots},
	height=5cm,
    width=4cm,
    ybar=1pt,
    area legend,
    xtick=data,
    tick label style={font=\scriptsize} ,
    symbolic x coords={Stocks, Houses, Currencies, GONU}, 
    ylabel style={align=center}, 
    legend columns = 3, 
    legend style={ legend, at={(-0.8,-0.15)},anchor=north},
    legend image post style={scale=0.4},
    legend to name=barlegend,
    area legend,
    legend entries={Spike and Slab,GPRN, LMC,MOSM, ABCD,R-ABCD,LKM (ours)}
]

\nextgroupplot[bar width=9pt]
\addplot[style={bblue,fill=bblue,mark=none}, error bars/.cd,y dir=both, y explicit, error bar style={color=black}] coordinates {(Stocks, 10.07) +- (0,0.12) };

\addplot[style={rred,fill=rred,mark=none}, error bars/.cd,y dir=both, y explicit, error bar style={color=black}]
coordinates {(Stocks,6.11) +- (0, 0.09)};
\addplot[style={GreenYellow,fill=GreenYellow,mark=none}, error bars/.cd,y dir=both, y explicit, error bar style={color=black}]
coordinates {(Stocks,8.20) +- (0, 0.53)};
\addplot[style={custompink,fill=custompink,mark=none}, error bars/.cd,y dir=both, y explicit, error bar style={color=black}]
coordinates {(Stocks,5.48) +- (0, 1.01)};
\addplot[style={ggreen,fill=ggreen,mark=none}, error bars/.cd,y dir=both, y explicit, error bar style={color=black}]
coordinates {(Stocks,8.37) +- (0,0.03) };
\addplot[style={ppurple,fill=ppurple,mark=none}, error bars/.cd,y dir=both, y explicit, error bar style={color=black}]
coordinates {(Stocks,4.88) +- (0, 0.03)};
\addplot[style={ppurple,fill=Tan,mark=none}, error bars/.cd,y dir=both, y explicit, error bar style={color=black}]
coordinates {(Stocks,4.58) +- (0, 0.16)};

\nextgroupplot[bar width=9pt]
\addplot[style={bblue,fill=bblue,mark=none}, error bars/.cd,y dir=both, y explicit, error bar style={color=black}]
coordinates {(Houses, 10.85) +- (0,0.46)};
\addplot[style={rred,fill=rred,mark=none}, error bars/.cd,y dir=both, y explicit, error bar style={color=black}]
coordinates { (Houses,8.96)+-(0, 0.17) };
\addplot[style={GreenYellow,fill=GreenYellow,mark=none}, error bars/.cd,y dir=both, y explicit, error bar style={color=black}]
coordinates { (Houses,11.31)+-(0, 1.04) };
\addplot[style={custompink,fill=custompink,mark=none}, error bars/.cd,y dir=both, y explicit, error bar style={color=black}]
coordinates {(Houses,8.15) +- (0, 1.51)};
\addplot[style={ggreen,fill=ggreen,mark=none}, error bars/.cd,y dir=both, y explicit, error bar style={color=black}]
coordinates {(Houses,7.98) +- (0, 0.03) };
\addplot[style={ppurple,fill=ppurple,mark=none}, error bars/.cd,y dir=both, y explicit, error bar style={color=black}]
coordinates { (Houses,3.17) +- (0, 0.10) };
\addplot[style={ppurple,fill=Tan,mark=none}, error bars/.cd,y dir=both, y explicit, error bar style={color=black}]
coordinates {(Houses,4.39) +- (0, 0.18)};

\nextgroupplot[bar width=9pt]
\addplot[style={bblue,fill=bblue,mark=none}, error bars/.cd,y dir=both, y explicit, error bar style={color=black}]
coordinates {(Currencies,174.71)+- (0,14.52)};
\addplot[style={rred,fill=rred,mark=none}, error bars/.cd,y dir=both, y explicit, error bar style={color=black}]
coordinates {(Currencies, 193.05)+- (0, 49.40)};
\addplot[style={GreenYellow,fill=GreenYellow,mark=none}, error bars/.cd,y dir=both, y explicit, error bar style={color=black}]
coordinates {(Currencies, 394.83)+- (0, 40.54)};
\addplot[style={custompink,fill=custompink,mark=none}, error bars/.cd,y dir=both, y explicit, error bar style={color=black}]
coordinates {(Currencies,318.26) +- (0, 101.52)};
\addplot[style={ggreen,fill=ggreen,mark=none}, error bars/.cd,y dir=both, y explicit, error bar style={color=black}]
coordinates {(Currencies,325.58)+-(0,8.64)};
\addplot[style={ppurple,fill=ppurple,mark=none}, error bars/.cd,y dir=both, y explicit, error bar style={color=black}]
coordinates {(Currencies, 208.62)+-(0,5.02)};
\addplot[style={ppurple,fill=Tan,mark=none}, error bars/.cd,y dir=both, y explicit, error bar style={color=black}]
coordinates { (Currencies, 133.00)+-(0, 16.92)};

\nextgroupplot[bar width=9pt]
\addplot[style={bblue,fill=bblue,mark=none}, error bars/.cd,y dir=both, y explicit, error bar style={color=black}]
coordinates {(GONU,1.07) +- (0, 0.08)};
\addplot[style={rred,fill=rred,mark=none}, error bars/.cd,y dir=both, y explicit, error bar style={color=black}]
coordinates {(GONU, 1.16)+- (0, 0.12)};
\addplot[style={GreenYellow,fill=GreenYellow,mark=none}, error bars/.cd,y dir=both, y explicit, error bar style={color=black}]
coordinates {(GONU, 1.01)+- (0, 0.14)};
\addplot[style={custompink,fill=custompink,mark=none}, error bars/.cd,y dir=both, y explicit, error bar style={color=black}]
coordinates {(GONU,0.84) +- (0, 0.18)};
\addplot[style={ggreen,fill=GreenYellow,mark=none}, error bars/.cd,y dir=both, y explicit, error bar style={color=black}]
coordinates {(GONU,0.86)+- (0, 0.01)};
\addplot[style={ppurple,fill=ppurple,mark=none}, error bars/.cd,y dir=both, y explicit, error bar style={color=black}]
coordinates {(GONU, 0.97)+- (0, 0.03) };
\addplot[style={ppurple,fill=Tan,mark=none}, error bars/.cd,y dir=both, y explicit, error bar style={color=black}]
coordinates { (GONU, 0.76) +- (0,0.07)};


\end{groupplot}
\node (fig3Legend) at ($(bar_plots c2r1.center)-(-1.5,2.9)$){\ref*{barlegend}};
\end{tikzpicture}
}
\captionof{figure}{RMSEs for each data set (9 stocks, 6 houses, 4 currencies, GONU) with corresponding methods.}
\label{fig:rmse_column}
\end{figure}
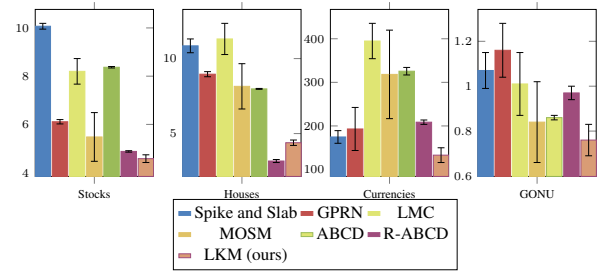

			
\textbf{Pairwise comparison} \hspace{5pt}
We provide another type of descriptive comparisons. Given a set of $N$ time series, the output of our model can generate ${N \choose 2}$ reports which compare each pair of time series. These reports give us a more detailed insight than the overview comparison. A report consists of shared components and individual ones between time series. Alongside with the description of the kernel structure of $\vC_k$, this type of report presents the corresponding posterior $\vf_k|\vx_n$ which will illustrate the variations of GP realizations on different time series (see Figure~\ref{fig:pair}).

  We bring a brief analysis of GONU data set as an example after taking a quick look over the generated report. For instance, the gold and oil prices share many common characteristics (long and short lengthscale varying), showing a marginally small difference. On the other hand, NASDAQ and USD indices differ each other with many distinctive individual kernels $\vC_k$s. Interestingly, the negative correlation behavior between the oil and USD indices (i.e. two time series often go in opposite directions) can be observed by shared kernels using~\acrshort*{lkm} (see Appendix~\ref{appendix:oil_usd}). These reports give an easy understanding for ones who do not have knowledge in finance. 

\subsection{Quantitative results}

\paragraph{Experiment setup} All experiments are conducted to predict future events (extrapolation) by splitting all data sets and trained with the first 90\%, then tested with the remaining 10\% as in the standard  setting for extrapolation tasks. Root mean square error (RMSE) and Mean Negative Log Likelihood (MNLP)~\cite{MNLP-Lazaro-Gredilla} are the main evaluation metrics in all data sets.

\textbf{Compare to multi-task~\glspl*{gp}}\hspace{5pt}  We compare multi-task GP models including `Spike and Slab' model \cite{spike_lab}  
, GP regression network (GPRN)~\cite{gprn_wilson12icml,nguyen13b_gprn}, Linear Model of Coregionalization (LMC)~\cite{kernel_vector,gpy2014} and Multi-Output Spectral Mixture (MOSM)~\cite{MOSM}. The result in Table~\ref{tab:rmse} and Figure~\ref{fig:rmse_column} indicates that our methods significantly outperform these models.
This result could be attributed to that~\acrshort*{lkm} leveraged by~\acrshort*{pe} selects compositional kernels which are flexible enough to fit complex data. \\ 

\textbf{Compare to existing kernel composition approaches} \hspace{5pt} We ran~\acrshort*{abcd} on individual time series then aggregated the results to compare with our models. Our model outperforms ABCD which is known as one of the state-of-the-art GP-based regression methods on univariate time series. It proves that our belief about the correlations among multiple time series is plausible. 
 
 We then compare with R-ABCD~\cite{Hwangb16}. Rather than making the assumption that all time series share a single global kernel, our model recognizes which structures are shared globally or partially. Quantitatively,~\acrshort*{lkm} shows promising results in prediction tasks. It outruns R-ABCD in most of the data sets (Table~\ref{tab:rmse} and Figure~\ref{fig:rmse_column}). In a relationally complex data set like GONU, LKM is significantly better while R-ABCD failed as the restriction due to its feature (function) sharing assumption.
 
Spike and Slab and GPRN models perform better than~\acrshort*{abcd} and R-ABCD in the currency data set where it contains highly volatile data. Although our model shares some computational procedures with~\acrshort*{abcd} and R-ABCD, our model is more robust to handle different types of time series data.


\section{Conclusion}
\label{sec:conclusion}
In this paper, we study a new perspective of multi-task GP learning where kernel structures are appropriately selected. We introduce the~\acrshort*{lkm} which learns kernel decompositions from a stochastic kernel process. We further present a pragmatic search algorithm leveraging our models to explore a larger structure space efficiently. Experimental results demonstrate promising performance in prediction tasks. Our proposed model also outputs a high-quality set of interpretable kernels which produces a comparison reports among multiple time series. 
\clearpage
\section*{Acknowledgment} 
This work is supported by Basic Science Research Program
through the National Research Foundation of Korea (NRF)
grant funded by the Korea government (MSIT: the Ministry
of Science and ICT) (NRF-2017R1A1A1A05001456) and
Institute for Information \& Communications Technology Planning \& Evaluation (IITP) grant funded by the MSIT (No.2017-0-01779, a machine learning and statistical inference framework
for explainable artificial intelligence).
\bibliography{ref}{}
\bibliographystyle{icml2019}

\clearpage
\onecolumn
\makeatother
\appendix

\section{Base kernels and search grammar in the ABCD framework}
\label{appendix:base}
Table~\ref{tab:base_kernel} contains base kernels described in~\cite{Lloyd2014ABCD}.

\begin{table*}
	\begin{center}
		\scalebox{1}{
			\begin{tabular}{l | l | c}
				Base Kernels & Encoding Function  & $k(x,x')$\\
				\hline
				White Noise (\textsc{WN}) & Uncorrelated noise & $\sigma^2 \delta_(x,x')$ \\
				Constant (\textsc{C}) & Constant functions & $\sigma^2$ \\
				Linear (\textsc{LIN}) & Linear functions & $\sigma^2 (x - l)(x - l')$\\
				Squared Exponential (\textsc{SE}) & Smooth functions & $\sigma^2 \exp (-\frac{(x - x')^2}{2l^2})$\\
				Periodic (\textsc{PER}) & Periodic functions & $\sigma^2 \frac{\exp(\frac{\cos \frac{2\pi(x-x')}{p}}{l^2}) - I_0(\frac{1}{l^2})}{\exp(\frac{1}{l^2}) - I_0(\frac{1}{l^2})}$ \\
			\end{tabular}
		}
	\end{center}
	\caption{List of base kernels}
	\label{tab:base_kernel}
\end{table*}

The language of models (or kernels) is presented by a set of rules in the grammar:
\begin{align*}
\mathcal{S} &\rightarrow \mathcal{S + B} &&\mathcal{S} \rightarrow \mathcal{S \times B} \\
\mathcal{S} &\rightarrow \mathcal{B}  
\end{align*}
where $\mathcal{S}$ represents any kernel subexpression, $\mathcal{B}$ and $\mathcal{B}'$ are base kernels~\cite{Lloyd2014ABCD}.

\section{Compute $\expect[\log p(\vZ)]$}
\label{appendix:expect}
In~\citet{vibp_Doshi}, the variational inference approximating $\vZ$ considered two approaches: finite variational approach and infinite variational approach. We will take a brief review of estimating $\expect[\log p(\vZ)]$. Readers may refer to~\cite{vibp_Doshi} to have more details.
In the finite variational approach, sampling $\vZ$ involves
\begin{align*}
	\pi_k \sim & \textrm{Beta}(\alpha/K,1),\\
	z_{nk} \sim & \textrm{Bernoulli}(\pi_k).
\end{align*}
Here the generative procedure involves an additional random variable $\boldsymbol{\pi}$ which is omitted in the main text for simplicity. The variational inference requires to approximate the posterior distribution over $\boldsymbol{\pi}$ by $\prod_k q(\pi_k)$. Specifically, each $q(\pi_k)$ follows a Beta distribution Beta$(\tau_{k_1}, \tau_{k_2})$. Since $\vX$ and $\boldsymbol{\pi}$ are conditionally independent given $\vZ$, $\expect[\log p(\vX|\vZ)]$ discussed in the main text is independent to $\boldsymbol{\pi}$. We can compute $\expect[\log p(\vZ)]$ as
\begin{align*}
	&\expect[\log p(\vZ)]\\
	 =& \sum_{k=1}^K\left[\log \frac{\alpha}{K} + \left(\frac{\alpha}{K} - 1\right)(\psi(\tau_{k_1})- \psi(\tau_{k_1} + \tau_{k_2}))\right] + \sum_{k=1}^K \sum_{n=1}^N [\nu_{nk}\psi(\tau_{k_1}) + (1 - \nu_{nk}) \psi(\tau_{k_2}) - \psi(\tau_{k_1} + \tau_{k_2})],
\end{align*}
where $\psi(\cdot)$ is the digamma function.

While in the finite variational approach, stick breaking construction~\cite{stick_breaking_ibp} is used to sample $\vZ$ as
\begin{align*}
v_k \sim & \textrm{Beta}(\alpha,1),\\
\pi_k = &\prod_{i=0}^K v_i, \\
z_{nk} \sim & \textrm{Bernoulli}(\pi_k),
\end{align*}
with $k=1\dots\infty$. Similarly, the variational distribution $q(\boldsymbol{v})$ is proposed to approximate $p(\boldsymbol{v})$ by independent Beta$(\tau_{k_1},\tau_{k_2})$s
\begin{align*}
\expect[\log p(\vZ)] =& \sum_{k=1}^K\left[\log {\alpha} + \left({\alpha} -1\right)(\psi(\tau_{k_1})- \psi(\tau_{k_1} + \tau_{k_2}))\right]  \\
&+\sum_{k=1}^K \sum_{n=1}^N \Biggl[\nu_{nk}\left(\sum_{m=1}^k\psi(\tau_{m_1}) - \psi(\tau_{m_1} + \tau_{m_2}) \right) + (1 - \nu_{nk})\expect_{\boldsymbol{v}}\Bigl[\log(1 - \prod_{m=1}^k v_m) \Bigr]  \Biggr],
\end{align*}
with $\expect_{\boldsymbol{v}}\left[\log(1 - \prod_{m=1}^k v_m) \right] $ is further approximated by Taylor expansion.

\section{Comparison of search space in~\acrshort*{pe} with~\acrshort*{lkm} and~\acrshort*{ckl}}
\label{sec:compare_pe_ckl}
We emphasize that~\acrshort*{pe} with~\acrshort*{lkm} considers a larger number of kernel structures than those in~\acrshort*{ckl}. 
Suppose that~\acrshort*{ckl} and our search algorithm have the same found structure at a depth $d$. While the~\acrshort*{ckl}'s structure is $\mathcal{S}_d = \mathcal{S}_d^{(1)}  + \dots + \mathcal{S}_d^{(K)}$,~\acrshort*{pe} represents as a set $\{\mathcal{S}_d^{(1)}, \dots,\mathcal{S}_d^{(K)}\}$. Let us examine the cardinality of kernel spaces after performing an expansion to the next depth.
The procedure is to extract substructures from the current structure, then apply grammar rules on the structure. In~\acrshort*{ckl}, substructures consist of all structures generated from the combinations of $\mathcal{B}_d^{(k_l)}$ in each individual $\mathcal{S}_d^{(k)}$ and ones generated by the combination of all $\mathcal{S}_d^{(k)}$. The former has $O(K \sum_l {L \choose l} ) = O(K2^L)$ substructures where $L$ is the largest number of base kernels in $\mathcal{S}_d^{(k)}$. The latter creates $O(\sum_{k}{K\choose k}) = O(2^K)$ combinations. When the maximum number of grammar rules per substructure is $R$, the total number of candidates at the depth $d+1$ is $O(RK2^L + R2^K)$.

Our approach only applies expansion on individual structure $\mathcal{S}_d^{(k)}$ via the combinations of $\mathcal{B}_d^{(k_l)}$. However, the search space still includes all the cases when substructures are extracted from a combination of $\mathcal{S}_{d}^{(k)}$. For instance, the generation from LIN+PER+SE to (LIN+PER)$\times$SE+SE in \acrshort*{ckl} is equivalent to the generation from $\{$LIN, PER, SE$\}$ to $\{$LIN$\times$SE, PER$\times$SE, SE$\}$ in our approach. 
For the case of PE, the additive kernel set will be expanded into a new one having the number of elements $R2^L + K$.  With the flexible binary indications (on/off) of $\vZ$, the number of all possible kernels is $O(K2^{R2^L + K})$ when all structures are visited to be expanded.

\begin{figure*}[t]
	\centering
	\includegraphics[width=0.8\textwidth]{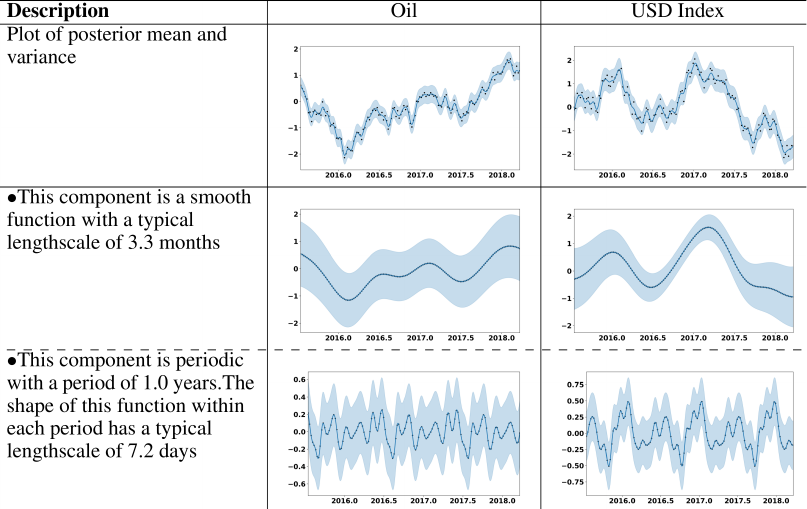}
	\caption{Comparing Oil and USD index. This is extracted from the pairwise comparison of GONU data set.}
	\label{fig:ou}
\end{figure*}

\section{Pairwise comparison between Oil and USD index}
As we discussed in the main text, our model can recognize the inverse correlation by looking at the first component in Figure~\ref{fig:ou}. The second component in Figure~\ref{fig:ou} is another example of the posterior $\vf_k|\vx_n$ realized differently given different time series. This observation is found in a real-world data set.
\label{appendix:oil_usd}
\section{Full output of seizure data set}
\label{sec:full_seizure}
Figure~\ref{fig:full_seizure} describes the output of our model.
\begin{figure*}[h!]
	\centering
	\includegraphics[width=0.65\textwidth]{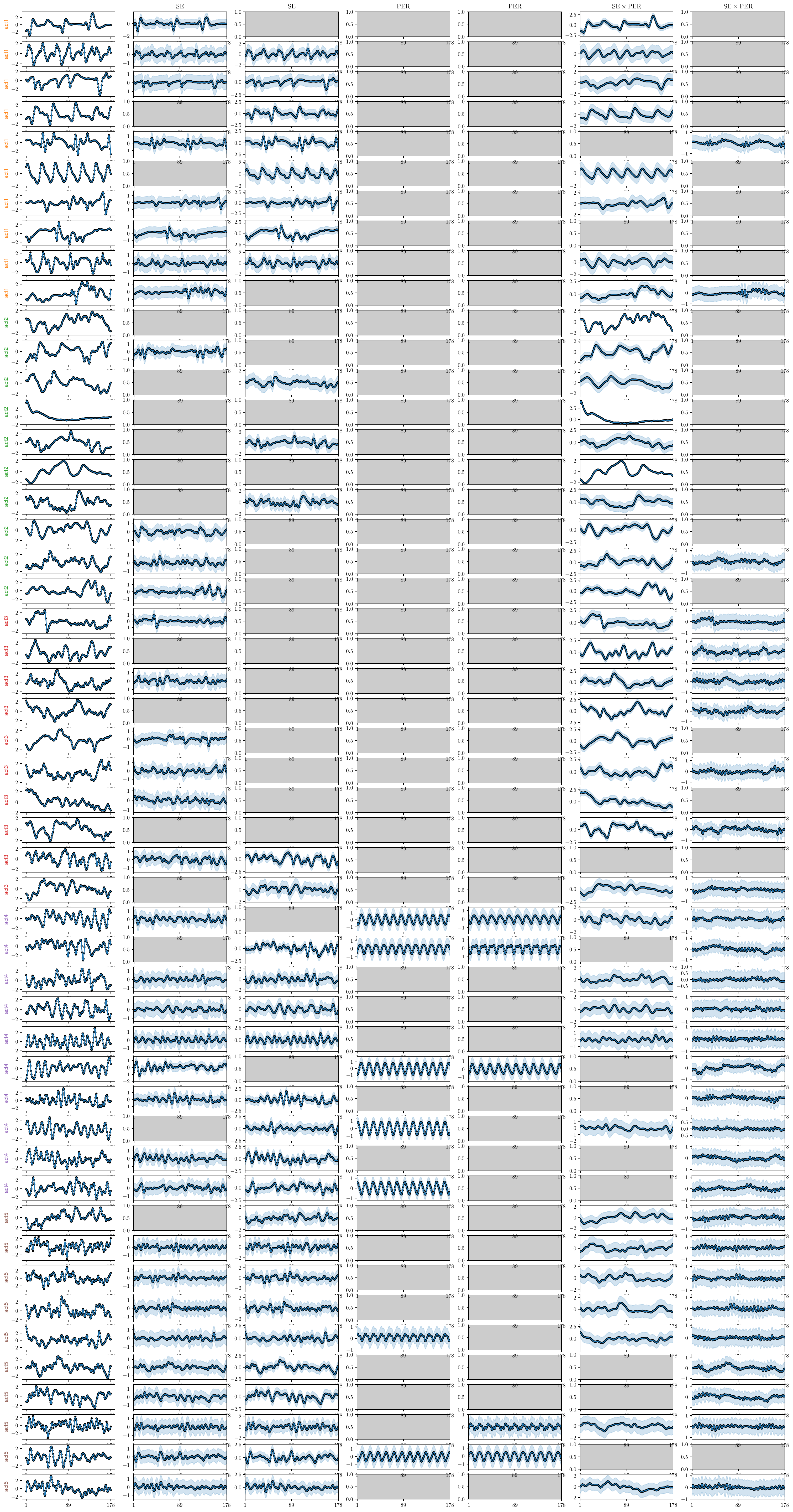}
	\caption{Full output of seizure data set. First column is the posterior plot of each time series. The remaining columns are decomposed components. Missing plot indicate there is no component w.r.t the corresponding time series.}
	\label{fig:full_seizure}
\end{figure*}

%
%

\end{document}